\documentclass{article}

\usepackage{amsmath}
\usepackage{amsfonts}
\usepackage{amsthm}
\usepackage{epsfig}
\usepackage{latexsym}
\usepackage{amsmath, amsthm, amssymb}
\usepackage{pseudocode}
\usepackage[whileod]{algochl}

\usepackage{amsfonts}
\usepackage{amsmath}
\usepackage{amssymb}
\usepackage{latexsym}
\usepackage{color}
\usepackage{url}
\usepackage{epsfig}
\usepackage{graphics}
\usepackage{enumerate}
\usepackage[sort]{natbib}

\newcommand{\bi}{\begin{itemize}}
\newcommand{\ei}{\end{itemize}}
\newcommand{\be}{\begin{enumerate}}
\newcommand{\ee}{\end{enumerate}}

\def\Rset{\mathbb{R}}

\newcommand{\nys}{Nystr\"{o}m}
\newcommand{\ignore}[1]{}

\providecommand{\norm}[1]{\lVert#1\rVert}

\providecommand{\pinv}[1]{#1^{+}}

\newcommand{\colspace}{@{\hspace{.1cm}}}

\newcommand{\tl}{\widetilde}
\newcommand{\mat}[1]{{\mathbf #1}}
\DeclareMathOperator*{\rank}{\mathrm rank}
\newcommand{\K}{\mat{K}}

\renewcommand{\P}{\mat{P}}
\newcommand{\T}{\mat{T}}
\newcommand{\X}{\mat{X}}

\newcommand{\U}{\mat{U}}
\newcommand{\V}{\mat{V}}
\newcommand{\W}{\mat{W}}

\newcommand{\I}{\mat{I}}

\newcommand{\mSigma}{\mat{\Sigma}}

\newcommand{\x}{\mat{x}}

\renewcommand{\v}{\mat{v}}
\newcommand{\z}{\mat{z}}
\newcommand{\evec}{\mat{e}}

\newcommand{\0}{\mat{0}}

\newtheorem{theorem}{Theorem}

\newtheorem{lemma}{Lemma}

\newtheorem{corollary}{Corollary}
\newtheorem{definition}{Definition}
\newcommand{\ipsfig}[2]{\scalebox{#1}{\psfig{#2}}}

\author{Mehryar Mohri\\
Courant Institute and Google Research\\
New York, NY\\
\texttt{mohri@cs.nyu.edu}
\and
Ameet Talwalkar\\
University of California, Berkeley\\
Berkeley, CA\\
\texttt{ameet@eecs.berkeley.edu}}

\title{\textbf{On the Estimation of Coherence}}

\date{}

\begin{document}
\maketitle

\begin{abstract}
Low-rank matrix approximations are often used to help scale standard machine
learning algorithms to large-scale problems.  Recently, matrix coherence has
been used to characterize the ability to extract global information from a
subset of matrix entries in the context of these low-rank approximations and
other sampling-based algorithms, e.g., matrix completion, robust PCA.  Since
coherence is defined in terms of the singular vectors of a matrix and is
expensive to compute, the practical significance of these results largely
hinges on the following question: \emph{Can we efficiently and accurately
estimate the coherence of a matrix?} In this paper we address this question.
We propose a novel algorithm for estimating coherence from a small number of
columns, formally analyze its behavior, and derive a new coherence-based matrix
approximation bound based on this analysis. We then present extensive
experimental results on synthetic and real datasets that corroborate our
worst-case theoretical analysis, yet provide strong support for the use of our
proposed algorithm whenever low-rank approximation is being considered.  Our
algorithm efficiently and accurately estimates matrix coherence across a wide
range of datasets, and these coherence estimates are excellent predictors of
the effectiveness of sampling-based matrix approximation on a case-by-case
basis.
\end{abstract}

\section{Introduction}
Large-scale datasets are becoming more and more prevalent for problems in a
variety of areas, e.g., computer vision, natural language processing,
computational biology. However, several standard methods in machine learning,
such as spectral clustering, manifold learning techniques, kernel ridge
regression or other kernel-based algorithms do not scale to such orders of
magnitude. For large datasets, these algorithms would require storage and
operation on matrices with thousands to millions of columns and rows, which is
especially problematic since these matrices are often not sparse.  An
attractive solution to such problems involves efficiently generating low-rank
approximations to the original matrix of interest.  In particular,
sampling-based techniques that operate on a subset of the columns of the matrix
can be effective solutions to this problem, and have been widely studied within
the machine learning and theoretical computer science communities
\citep{Frieze98,Drineas06,Williams00,Kumar09b}.  In the context of kernel
matrices, the \nys\ method \citep{Williams00} has been shown to work
particularly well in practice for various applications ranging from manifold
learning to image segmentation \citep{Fowlkes04,Talwalkar08}.  

A crucial assumption of these algorithms involves their sampling-based nature,
namely that an accurate low-rank approximation of some matrix $\X \in \Rset^{n
\times m}$ can be generated exclusively from information extracted from a small
subset ($l \ll m$) of its columns.  This assumption is not generally true for
all matrices, and explains the negative results of \citet{Fergus09}.  For
instance, consider the extreme case:
\begin{equation} 
\label{eq:bad_example}
\X = \left[ \begin{array}{cccccc}
\Big | & & \Big | & \Big | &  & \Big | \\ 
\evec_1 & \ldots & \evec_r & \0 & \ldots & \0 \\
\Big | & & \Big | & \Big | &  & \Big | \\ 
\end{array} \right],
\end{equation}
where $\evec_i$ is the $i$th column of the $n$ dimensional identity matrix and
$\0$ is the $n$ dimensional  zero vector.  Although this matrix has rank $r$,
it cannot be well approximated by a random subset of $l$ columns unless this
subset includes $\evec_1, \ldots, \evec_r$. In order to account for such
pathological cases, previous theoretical bounds relied on sampling columns of
$\X$ in an adaptive fashion \citep{Smola00,Bach05,Deshpande06,Kumar09b} or from
non-uniform distributions derived from properties of $\X$
\citep{Drineas05,Drineas06}. Indeed, these bounds give better guarantees for
pathological cases, but are often quite loose nonetheless, e.g., when dealing
with kernel matrices using RBF kernels, and these sampling schemes are rarely
utilized in practice.

More recently, \citet{Talwalkar10} used the notion of \emph{coherence} to
characterize the ability to extract information from a small subset of columns,
showing theoretical and empirical evidence that coherence is tied to the
performance of the \nys\ method. Coherence measures the extent to which the
singular vectors of a matrix are correlated with the standard basis.
Intuitively, if the dominant singular vectors of a matrix are incoherent, then
the subspace spanned by these singular vectors is likely to be captured by a
random subset of sampled columns of the matrix. In fact, coherence-based
analysis of algorithms has been an active field of research, starting with
pioneering work on compressed sensing \citep{Donoho06,Candes06}, as well as
related work on matrix completion \citep{candes08,Keshavan09b} and robust
principle component analysis \citep{candes09}. 

In \citet{candes08}, the use of coherence is motivated by results showing that
several classes of randomly generated matrices have low coherence with high
probability, one of which is the class of matrices generated from uniform
random orthonormal singular vectors and arbitrary singular values.
Unfortunately, these results do not help a practitioner compute coherence on a
case-by-case basis to determine whether attractive theoretical bounds hold for
the task at hand. Furthermore, the coherence of a matrix is by definition
derived from its singular vectors and is thus expensive to compute (the
prohibitive cost of calculating singular values and singular vectors is
precisely the motivation behind sampling-based techniques).  Hence, in spite of
the numerous theoretical work based on related notions of coherence, the
practical significance of these results largely hinges on the following open
question: \emph{Can we efficiently and accurately estimate the coherence of a
matrix?}

In this paper we address this question by presenting a novel algorithm for
estimating matrix coherence from a small number of columns. The remainder of
this paper is organized as follows. Section \ref{sec:prelim} introduces basic
definitions, and provides a brief background on low-rank matrix approximation
and matrix coherence.  In Section \ref{sec:algorithm} we introduce our
sampling-based algorithm to estimate matrix coherence.  We then formally
analyze its behavior in Section \ref{sec:theory}, and also use this analysis to
derive a novel coherence-based bound for matrix projection reconstruction via
Column-sampling (defined in Section \ref{ssec:low_rank}).  Finally, in Section
\ref{sec:experiments} we present extensive experimental results on synthetic
and real datasets.  These results corroborate our worst-case theoretical
analysis, yet provide strong support for the use of our proposed algorithm
whenever sampling-based matrix approximation is being considered.  Empirically,
our algorithm effectively estimates matrix coherence across a wide range of
datasets, and these coherence estimates are excellent predictors of the
effectiveness of sampling-based matrix approximation on a case-by-case basis.
 
\section{Background}
\subsection{Notation}
\label{sec:prelim}
Let $\X \in \Rset^{n \times m}$ be an arbitrary matrix.  We define $\X^{(j)},
\, j = 1 \ldots m$, as the $j$th column vector of $\X$, $\X_{(i)}, \, i = 1
\ldots n$, as the $i$th row vector of $\X$ and $\X_{ij}$ as the $ij$th entry of
$\X$. We denote by $\norm{\X}_F$ the Frobenius norm of $\X$ and by $\norm{\v}$
the $l_2$ norm of the vector $\v$.  If $\rank(\X) = r$, we can write the thin
Singular Value Decomposition (SVD) as $\X = \U_{X} \mSigma_{X} \V_{X}^\top$.
$\mSigma_X$ is diagonal and contains the singular values of $\X$ sorted in
decreasing order, i.e., $\sigma_1(\X) \ge \sigma_2(\X) \ge \ldots \ge
\sigma_r(\X)$. $\U_X \in \Rset^{n \times r}$ and $\V_X \in \Rset^{m \times r}$
have orthogonal columns that contain the left and right singular vectors of
$\X$ corresponding to its singular values.  We define $\P_X = \U_X \U_X^\top$
as the orthogonal projection matrix onto the column space of $\X$, and denote
the projection onto its orthogonal complement as $\P_{X,\perp} = \I - \P_X$.
We further define $\pinv{\X} \in \Rset^{m\times n}$ as the Moore-Penrose
pseudoinverse of $\X$, with $\pinv{\X} =\V_{X} \pinv{\mSigma_{X}} \U_{X}^\top$.
Finally, we will define $\K \in \Rset^{n \times n}$ as a symmetric positive
semidefinite (SPSD) matrix with $\rank(\K) =r \le n$, i.e. a symmetric matrix
with non-negative eigenvalues. 

\subsection{Low-rank matrix approximation}
\label{ssec:low_rank}

Starting with an $n \times m$ matrix, $\X$, we are interested in algorithms
that generate a low-rank approximation, $\tl \X$, from a sample of $l \ll n$ of
its columns.\ignore{\footnote{Sampling is an important issue discussed in
several previous work, e.g., \citet{Talwalkar10b}.}} The accuracy of this
approximation is often measured using the Frobenius or Spectral distance, i.e.,
$\norm{\X - \tl \X}_{\{2,F\}}$. We next briefly describe two of the most common
algorithms of this form, the Column-sampling and the \nys\ methods.

The Column-sampling method generates approximations to arbitrary rectangular
matrices. We first sample $l$ columns of $\X$ such that $\X = \begin{bmatrix}
\X_1 & \X_2\\ \end{bmatrix}$, where $\X_1$ has $l$ columns, and then use the
SVD of $\X_1$, $\X_1 =\U_{X_1} \mSigma_{X_1} \V_{X_1}^\top$, to approximate the
SVD of $\X$ \citep{Frieze98}.  This method is most commonly used to generate a
`matrix projection' approximation \citep{Kumar09b} of $\X$ as follows:
\begin{equation}
\label{eq:mat_proj_defined}
\tl \X^{col} = \U_{X_1}\U_{X_1}^\top \X.
\end{equation}
The runtime of the Column-sampling method is dominated by the SVD of $\X_1$
which takes O($nl^2$) time to perform and is feasible for small $l$.

In contrast to the Column-sampling method, the \nys\ method deals only with
SPSD matrices.  We start with an $n \times n$ SPSD matrix, sampling $l$ columns
such that $\K = \begin{bmatrix} \K_1 & \K_2\\ \end{bmatrix}$, where $\K_1$ has $l$
columns, and define $\W$ as the $l \times l$ matrix consisting of the
intersection of these $l$ columns with the corresponding $l$ rows of $\K$.
Since $\K$ is SPSD, $\W$ is also SPSD. Without loss of generality, we can
rearrange the columns and rows of $\K$ based on this sampling such that:
\begin{equation}
\label{eq:blockK}
\K = 
\begin{bmatrix}
      \W & \widehat{\K}_1^\top\\
      \widehat{\K}_{1} & \widehat{\K}_{2}
\end{bmatrix}
\quad \text{where} \quad
\K_1 = 
\begin{bmatrix}
\W \\
\widehat{\K}_{1}
\end{bmatrix}
\quad \text{and} \quad
\K_2 = 
\begin{bmatrix}
\widehat{\K}_{1}^\top \\
\widehat{\K}_{2}
\end{bmatrix},
\end{equation}
The \nys\ method uses $\W$ and $\K_1$ from (\ref{eq:blockK}) to generate a
`spectral reconstruction' \citep{Kumar09b} approximation of $\K$ as $\tl
\K^{nys} = \K_1 \pinv{\W} \K_1^\top$.
Since the running time complexity of SVD on $\W$ is in $O(l^3)$ and matrix
multiplication with $\K_1$ takes $O(l^2n)$, the total complexity of the \nys\
approximation computation is also in $O(l^2n)$.

\subsection{Matrix Coherence}
\label{ssec:coherence}

Matrix coherence measures the extent to which the singular vectors of a matrix
are correlated with the standard basis.  As previously mentioned, coherence has
been to analyze techniques such as compressed sensing, matrix completion,
robust PCA, and the \nys\ method.  These analyses have used a variety of
related notions of coherence.  If we let $\evec_i$ be the $i$th column of the
standard basis, we can define three basic notions of coherence as follows:

\begin{definition}[$\mu$-Coherence]
\label{eq:coherence}
Let $\U \in \Rset^{n \times r}$ contain orthonormal columns with $r < n$.
Then the $\mu$-coherence of $\U$ is:
\begin{equation}
\mu(\U) = \sqrt{n} \max_{i,j} \big |{\U}_{ij} \big | \;.
\end{equation} 
\end{definition}

\begin{definition}[$\mu_0$-Coherence]
\label{eq:coherence_0}
Let $\U \in \Rset^{n \times r}$ contain orthonormal columns with $r < n$ and
define $\P_U = \U \U^\top$ as its associated orthogonal projection matrix. Then
the $\mu_0$-coherence of $\U$ is:
\begin{equation}
\mu_0(\U) = \frac{n}{r} \max_{1 \le i \le n} \norm{\P_U \evec_i}^2 = \max_{1 \le i \le n} \norm{\U_{(i)}}^2  \,. 
\end{equation} 
\end{definition}

\begin{definition}[$\mu_1$-Coherence]
\label{eq:coherence_1}
Given the matrix $\X \in \Rset^{n \times m}$ with rank $r$, left and right
singular vectors, $\U_X$ and $\V_X$, and define $\T =
\sum_{1 \le k \le r} \U_X^{(k)} {\V_X^{(k)}}^\top$.  Then, the
$\mu_1$-coherence of $\X$ is: 
\begin{equation}
\mu_1(\X) = \sqrt{\frac{nm}{r}} \max_{ij} \big |{\T}_{ij} \big | \,. 
\end{equation} 
\end{definition}

In \citet{Talwalkar10}, $\mu(\U)$ is used to provide coherence-based bounds for
the \nys\ method, where $\U$ corresponds to the singular vectors of a low-rank
SPSD kernel matrix. Low-rank matrices are also the focus of work on matrix
completion by \citet{candes08} and \citet{Keshavan09b}, though they deal with
more general rectangular matrices with SVD $\X = \U_{X} \mSigma_{X}
\V_{X}^\top$, and they use $\mu_0(\U_X)$, $\mu_0(\V_X)$ and $\mu_1(\X)$ to
bound the performance of two different matrix completion algorithms. Note that
a stronger, more complex notion of coherence is used in \citet{candes08b} to
provide tighter bounds for the matrix completion algorithm presented in
\citet{candes08} (definition omitted here).  Moreover, coherence has also been
used to analyze algorithms dealing with low-rank matrices in the presence of
noise, e.g., \citet{Keshavan09,candes09b} for noisy matrix completion and
\citet{candes09} for robust PCA. In these analyses, the coherence of the
underlying low-rank matrix once again appears in the form of $\mu_0(\cdot)$ and
$\mu_1(\cdot)$.

In this work we choose to focus on $\mu_0$.  In comparison to $\mu$, $\mu_0$ is
a more robust measure of coherence, as it deals with row norms of $\U$, rather
than the maximum entry of $\U$, and the two notions are related by a simple
pair of inequalities: $\mu^2 / r \le \mu_0 \le \mu^2$.
\ignore{Hence the bounds in \citet{Talwalkar10} also hold for $\mu_0$.}
Furthermore, since we focus on coherence in the context of algorithms that
sample columns of the original matrix, $\mu_0$ is a more natural choice than
$\mu_1$, since existing coherence-based bounds for these algorithms (both in
\citet{Talwalkar10} and in Section \ref{sec:theory} of this work) only
depend on the left singular vectors of the matrix.

\section{\textsc{Estimate-Coherence} Algorithm}
\label{sec:algorithm}
As discussed in the previous section, matrix coherence has been used to analyze
a variety of algorithms, under the assumption that the input matrix is either
exactly low-rank or is low-rank with the presence of noise.  In this section,
we present a novel algorithm to estimate the coherence of matrices following
the same assumption.  Starting with an arbitrary $n \times m$ matrix, $\X$, we
are ultimately interested in an estimate of $\mu_0(\U_X)$, which contains the
scaling factor $n / r$ as shown in Definition \ref{eq:coherence_0}.  However,
our estimate will also involve singular vectors in dimension $n$, and as we
mentioned above, $r$ is assumed to be small.  Hence, neither of these scaling
terms has a significant impact on our estimation.  As such, our algorithm
focuses on the closely related expression:
\begin{equation}
\label{eq:gamma}
\gamma(\U) = \max_{1 \le i \le n} \norm{\P_U \evec_i}^2 = \frac{r}{n} \mu_0\,.
\end{equation}

Our proposed algorithm is quite similar in flavor to the Column-sampling
algorithm discussed in Section \ref{ssec:low_rank}. It estimates
coherence by first sampling $l$ columns of the matrix and subsequently using
the left singular vectors of this submatrix to obtain an estimate.  Note that
our algorithm applies both to exact low-rank matrices as well as low-rank
matrices perturbed by noise. In the latter case, the algorithm requires a
user-defined low-rank parameter, $r$.  The runtime of this algorithm is
dominated by the singular value decomposition of the $n \times l$ submatrix,
and hence is in O($l^2n)$. The details of the \textsc{Estimate-Coherence}
algorithm are presented in Figure \ref{fig:coh_algorithm}.  

\begin{figure}[ht!]
\textbf{Input}: $n \times l$ matrix ($\X_1$) storing $l$ columns of arbitrary $n \times m$
matrix $\X$, low-rank parameter ($r$) \\
\textbf{Output}: An estimate of the coherence of $\X$ 
\begin{ALGO}{Estimate-Coherence}{\X_1, r}
\SET{\U_{X_1} }{\Call{SVD}{\X_1} 
\ \ $\COMMENT{keep left singular vectors}$}  
\SET{q}{\min\big(\rank(\X_1),r\big)}
\SET{\tl\U}{\Call{Truncate}{\U_{X_1},q} 
\ \ $\COMMENT{keep top $q$ singular vectors of $\X_1$}$}
\SET{\gamma(\X_1)}{\Call{Calculate-Gamma}{\tl\U}
\ \ $\COMMENT{see equation (\ref{eq:gamma})}$ }
\RETURN{\gamma(\X_1)}
\end{ALGO}
\caption{The proposed sampling-based algorithm to estimate matrix coherence.
Note that $r$ is only required when $\X$ is perturbed by noise.}
\label{fig:coh_algorithm}
\end{figure}

\section{Theory}
\label{sec:theory}
In this section we present a theoretical analysis of
\textsc{Estimate-Coherence} when used with low-rank matrices. Our main
theoretical results are presented in Theorem \ref{thm:est_coh}.

\begin{theorem}
\label{thm:est_coh}
Define $\X \in \Rset^{n \times m}$ with $\rank(\X) = r \ll n$, and denote by
$\U_X$ the $r$ left singular vectors of $\X$ corresponding to its non-zero
singular values.  Let the orthogonal projection onto span($\X_1$) be denoted by
$\P_{X_1} = \U_{X_1}\U_{X_1}^\top$, and define the projection onto its
orthogonal complement as $\P_{X_1,\perp}$. Let $\X_1$ be a set of $l$ columns
of $\X$ sampled uniformly at random, and let $\x$ be a column of $\X$ that is
not in $\X_1$ that is sampled uniformly at random.  Then the following
statements can be made about $\gamma(\X_1)$, which is the output of
\textsc{Estimate-Coherence}($\X_1$): 

\begin{enumerate}
\item $\gamma(\X_1)$ is a monotonically increasing estimate of $\gamma(\X)$.
Furthermore, if $\X'_1 = \begin{bmatrix} \X_1 & \x\\ \end{bmatrix}$ with
$\x_\perp = \P_{X_1,\perp} \x$, then $0
\le \gamma(\X'_1) - \gamma(\X_1) \le \gamma(\z)$, where $\z = \x_\perp 
/ \norm{\x_\perp}$.
\item $\gamma(\X_1) = \gamma(\X)$ when $\rank(\X_1) = \rank(\X)$, and the
probability of this event is dependent on the coherence of $\X$.  Specifically,
for any $\delta > 0$, it occurs with probability $1-\delta$ for $l \geq r^2
\mu_0(\U_X) \max \big(C_1 \log(r), C_2 \log(3 / \delta ) \big)$ for positive
constants $C_1$ and $C_2$.
\end{enumerate}
\end{theorem}

The second statement in Theorem \ref{thm:est_coh} leads to Corollary
\ref{cor:mat_proj}, which relates matrix coherence to the performance of the
Column-sampling algorithm when used for matrix projection on a low-rank matrix.
\begin{corollary}
\label{cor:mat_proj}
Assume the same notation as defined in Theorem \ref{thm:est_coh}, and let $\tl
\X^{col}$ be the matrix projection approximation generated by the
Column-sampling method using $\X_1$, as described in
(\ref{eq:mat_proj_defined}).  Then, for any $\delta > 0$, $\tl \X^{col} = \X$
with probability $1-\delta$, for $l \geq r^2 \mu_0(\U_X) \max \big(C_1 \log(r),
C_2 \log(3 / \delta ) \big)$ for positive constants $C_1$ and $C_2$.
\end{corollary}
\begin{proof}
When $\rank(\X_1) = \rank(\X)$, the columns of $\X_1$ span the columns of $\X$.
Hence, when this event occurs, projecting $\X$ onto the span of the columns of
$\X_1$ leaves $\X$ unchanged. The second statement in Theorem \ref{thm:est_coh}
bounds the probability of this event.
\end{proof}

\subsection{Proof of Theorem \ref{thm:est_coh}}
We first present Lemmas \ref{lemma_proj} and \ref{lemma_rankC}, and then
complete the proof of Theorem \ref{thm:est_coh} using these lemmas.

\begin{lemma}
\label{lemma_proj}
Assume the same notation as defined in Theorem \ref{thm:est_coh}.  Further, let
$\P_{X'_1}$ be the orthogonal projection onto span($\X'_1$) and define $s =
\norm{\x_\perp}$. Then, for any $l \in [1, n - 1]$, the following
equalities relate the projection matrix $\P_{X'_1}$ to $\P_{X_1}$:
\begin{equation}
  \P_{X'_1} =
\begin{cases}
\P_{X_1} +  \z\z^\top& \text{if $s > 0$};\\
\P_{X_1} & \text{if $s = 0$}.
\end{cases}
\end{equation}
\end{lemma}
\begin{proof}
First assume that $s = 0$, which implies that $\x$ is in the span of the
columns of $\X_1$.  Since orthogonal projections are unique, then clearly
$\P_{X'_1} = \P_{X_1}$ in this case. Next, assume that $s>0$, in which case the
span of the columns of $\X'_1$ can be viewed as the subspace spanned by the
columns of $\X_1$ along with the subspace spanned by the residual of $\x$, i.e.,
$\x_{\perp}$. Observe that $\z\z^\top$ is the orthogonal projection onto
span($\x_\perp$). Since these two subspaces are orthogonal and since orthogonal
projection matrices are unique, we can write $\P_{X'_1}$ as the sum of
orthogonal projections onto these subspaces, which matches the statement of the
lemma for $s>0$. 
\end{proof}

\begin{lemma}
\label{lemma_rankC}
Assume the same notation as defined in Theorem \ref{thm:est_coh}.  Then, if $l
\geq r^2 \mu_0(\U_X) \max \big(C_1 \log(r), C_2 \log(3 / \delta ) \big),$ where
$C_1$ and $C_2$ are positive constants, then for any $\delta > 0$, with
probability at least $1-\delta$, $\rank(\X_1) = r$.
\end{lemma}
\begin{proof}
Assuming uniform sampling at random, \citet{Talwalkar10} shows that
$\Pr[\rank(\X_1) = r] \ge \Pr \big (\|c \V_{X,l}^\top \V_{X,l} - \I\|_2 < 1
\big )$ for any $c \ge 0$, where $\V_{X,l} \in \Rset^{l \times r}$ corresponds
to the first $l$ components of the $r$ right singular vectors of $\X$. Applying
Theorem $1.2$ in \citet{Candes2007} and using the identity $r \mu_0 \ge \mu^2$
yields the statement of the lemma.
\end{proof} 

Now, to prove Theorem \ref{thm:est_coh} we analyze the difference:
\begin{align}
\label{eq:lips2}
\Delta_l & = \big | \gamma(\X_1') - \gamma(\X_1)\big | = \Big | \max_{j}
\evec_j^\top \P_{X'_1} \evec_j - \max_{i} \evec_i^\top \P_{X_1} \evec_i  \Big |
\,.
\end{align}
If $s = \norm{\x_\perp} = 0$, then by Lemma \ref{lemma_proj}, $\Delta_l = 0$. If
$s > 0$, then using Lemma \ref{lemma_proj} and (\ref{eq:lips2}) yields: 
\begin{align}
\label{eq:lips3}
\Delta_l & = \max_{j} \evec_j^\top \big
(\P_{X_1} + \z\z^\top \big ) \evec_j - \max_{i} \evec_i^\top \P_{X_1} \evec_i \\
\label{eq:lips4}
& \le \max_{j} \evec_j^\top \z\z^\top \evec_j = \gamma(\z).
\end{align}
In (\ref{eq:lips3}), we use the fact that orthogonal projections are always
SPSD, which means that $\evec_j^\top \z \z^\top \evec_j \ge 0$ for all $j$ and
ensures that $\Delta_l \ge 0$. In (\ref{eq:lips4}) we decouple
the $\max(\cdot)$ over $\P_{X_1}$ and $\z\z^\top$ to obtain the inequality and
then apply the definition of $\gamma(\cdot)$, which yields the first statement
of Theorem \ref{thm:est_coh}.  Finally, the second statement of Theorem
\ref{thm:est_coh} follows directly from Lemma \ref{lemma_proj} when $s = 0$
along with Lemma \ref{lemma_rankC}, as the former shows that $\Delta_l = 0$ if
$\rank(\X_1) = \rank(\X)$ and the latter gives a coherence-based finite-sample
bound on the probability of this event occurring. 

\section{Experiments}
\label{sec:experiments}
Theorem \ref{thm:est_coh} suggests that the ability to estimate matrix
coherence is dependent on the coherence of the matrix itself. In fact, if we
adversarially construct a high coherence matrix and select columns from this
matrix in an unfortunate manner, the results are quite discouraging.  For
instance, imagine that we generate a random SPSD matrix, e.g., using the
\textsc{Rand} function in Matlab, and then replace its first diagonal with an
arbitrarily large value, leading to a very high coherence matrix.  If we
subsequently force our sampling mechanism to ignore the first column of this
matrix, we are completely unable to estimate coherence using
\textsc{Estimate-Coherence}, as illustrated in Figure \ref{fig:worstcase} on a
synthetic matrix generated in Matlab following this procedure, with $n=1000$
and $k=50$.
\begin{figure}[ht!]
\centering
\ipsfig{.35}{figure=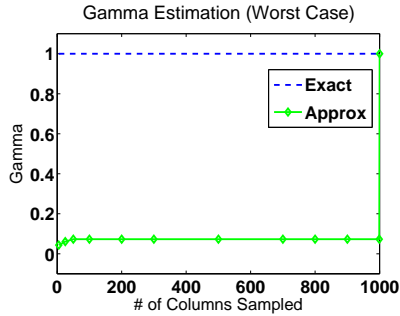}
\caption{Synthetic dataset illustrating worst-case performance
of \textsc{Estimate-Coherence}. }
\label{fig:worstcase}
\end{figure}

In spite of these discouraging worst-case results, our extensive empirical
studies show that \textsc{Estimate-Coherence} performs quite well in practice
on a variety of synthetic and real datasets with varying coherence, 
suggesting that the worst case addressed in theory and matched empirically in
Figure \ref{fig:worstcase} is rarely encountered in practice.  We present
these results in the remainder of this section.
\begin{figure}[ht!]
\begin{center}
\begin{tabular} {@{}c@{}c@{}}
\ipsfig{.35}{figure=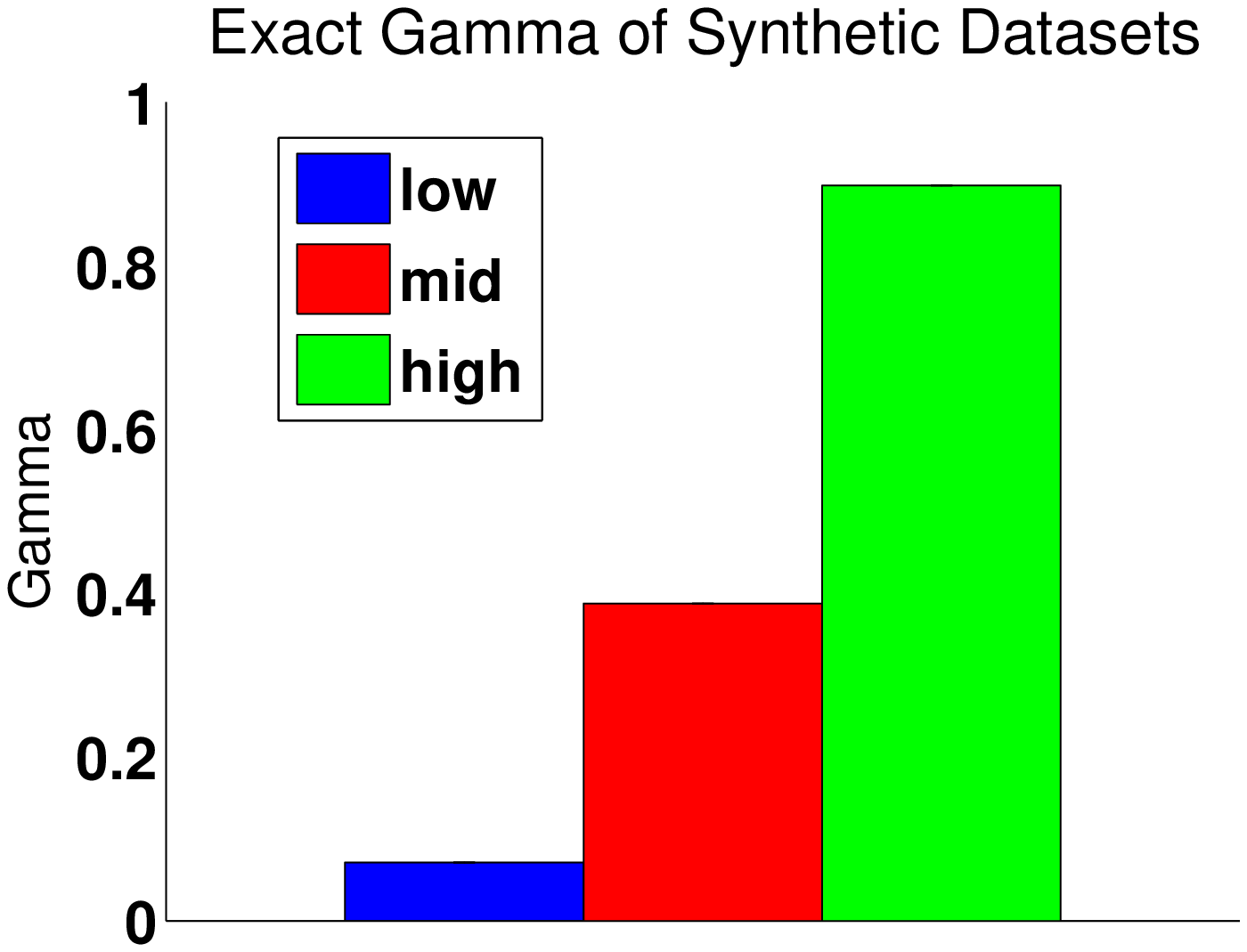} &
\ipsfig{.35}{figure=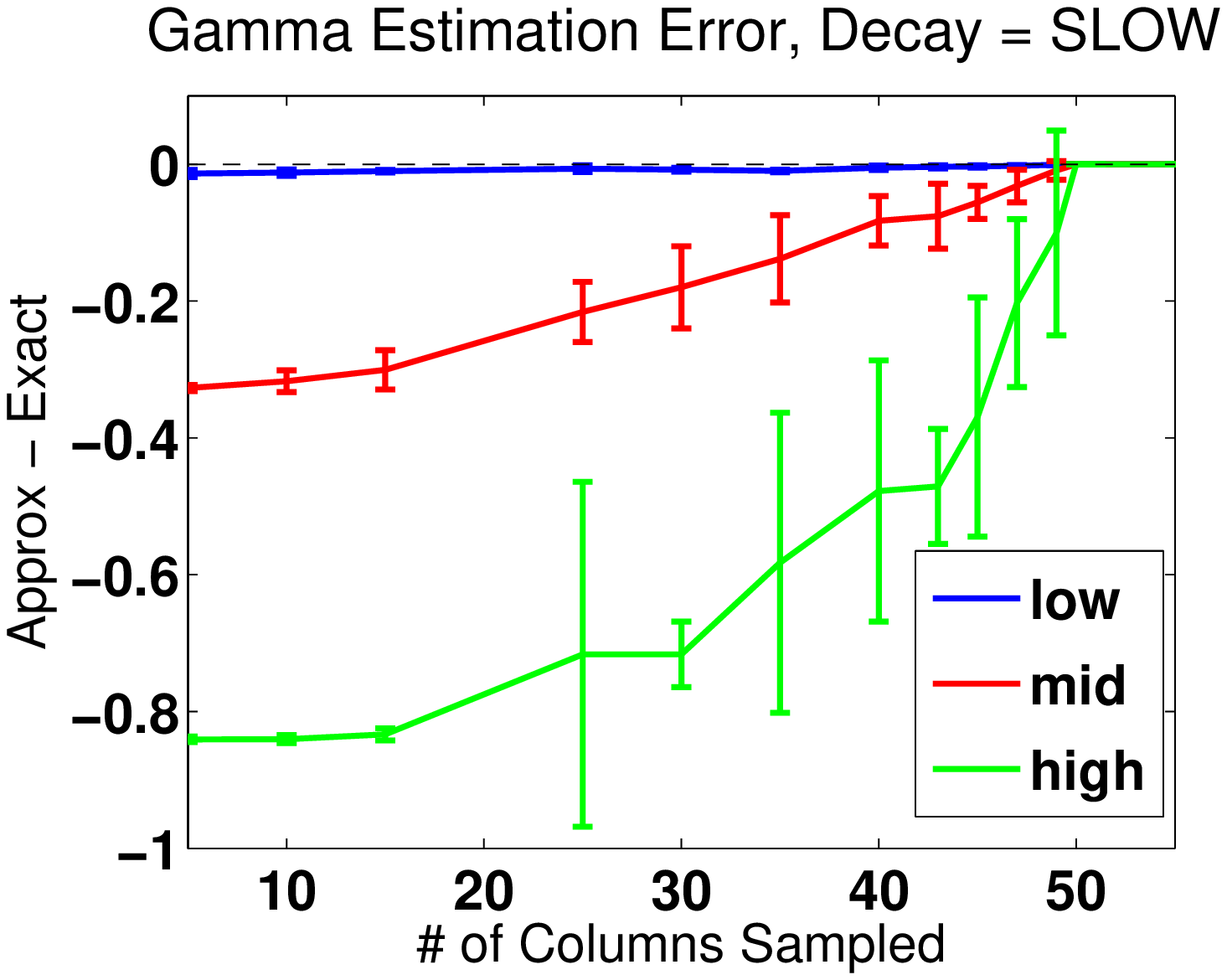} \\ 
(a) & (b) \\
\ipsfig{.35}{figure=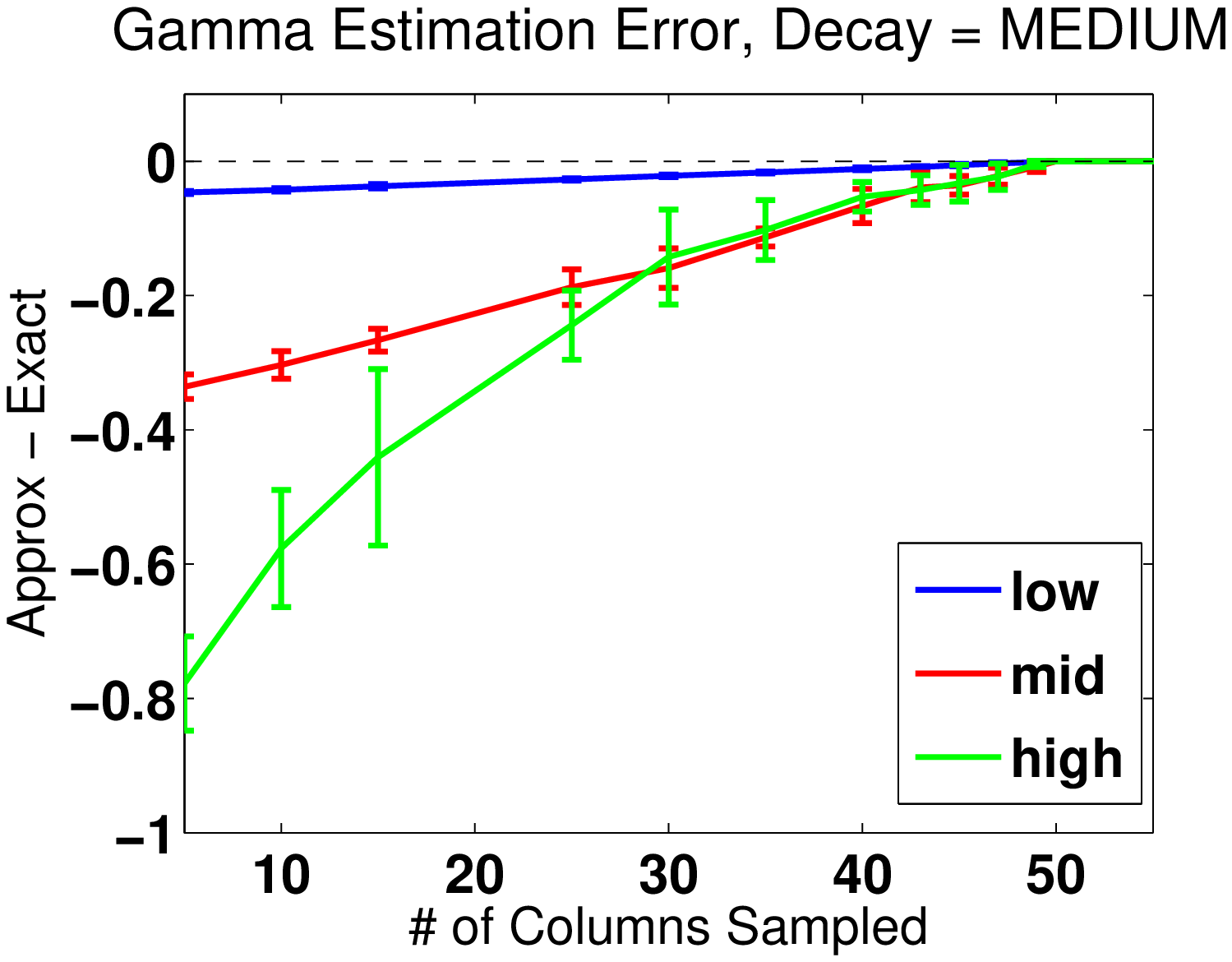} & 
\ipsfig{.35}{figure=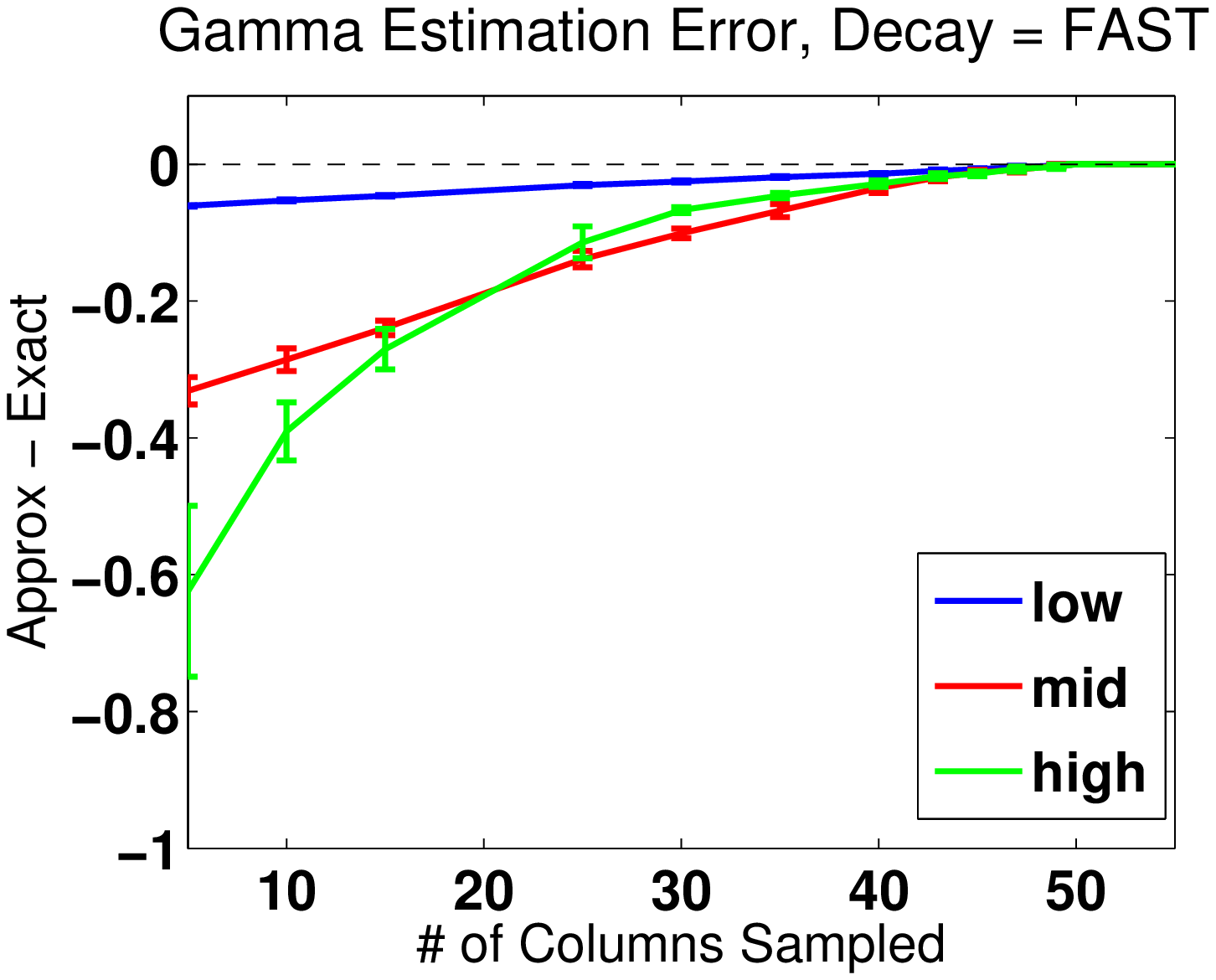} \\ 
(c) & (d) \\
\ipsfig{.35}{figure=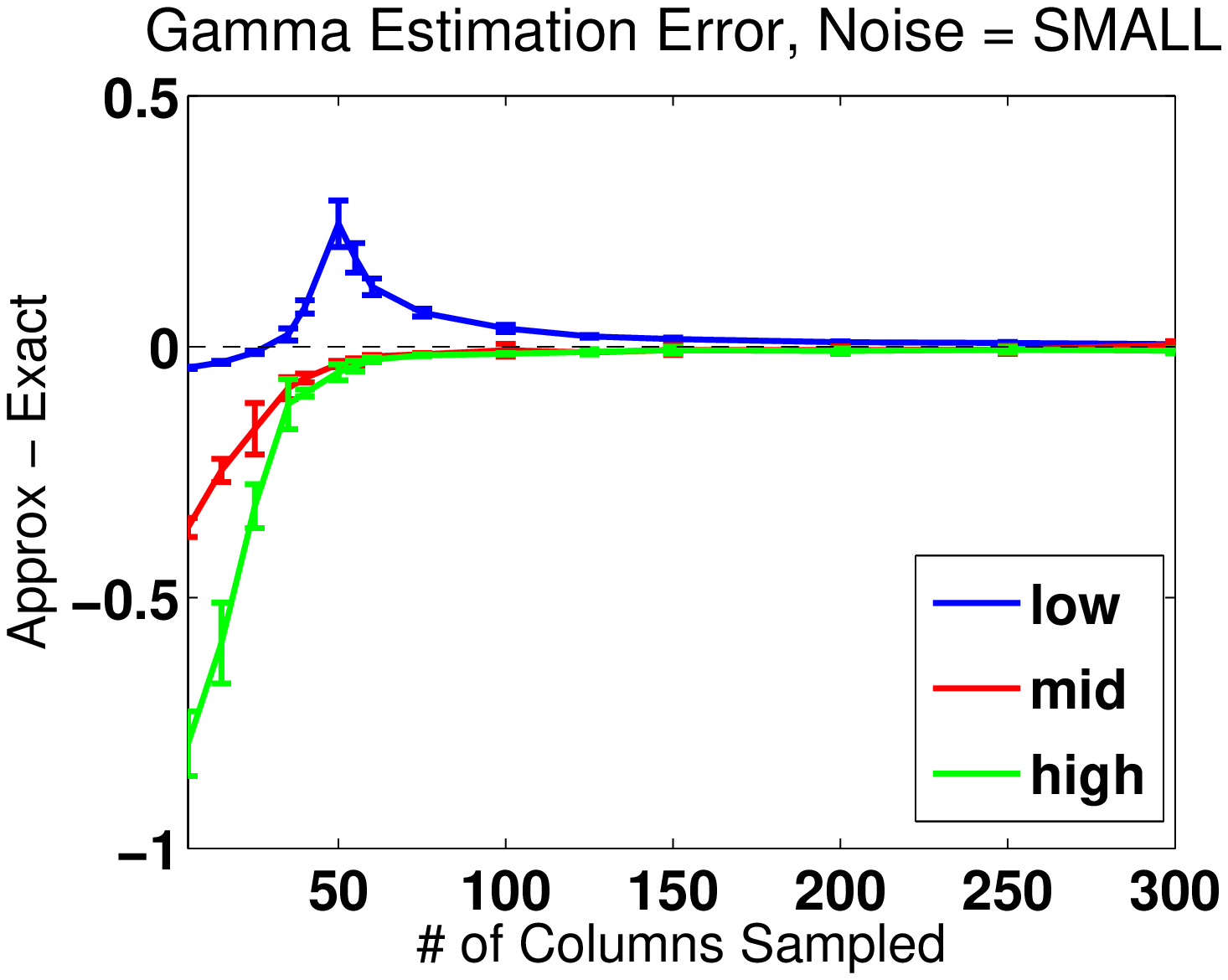} &
\ipsfig{.35}{figure=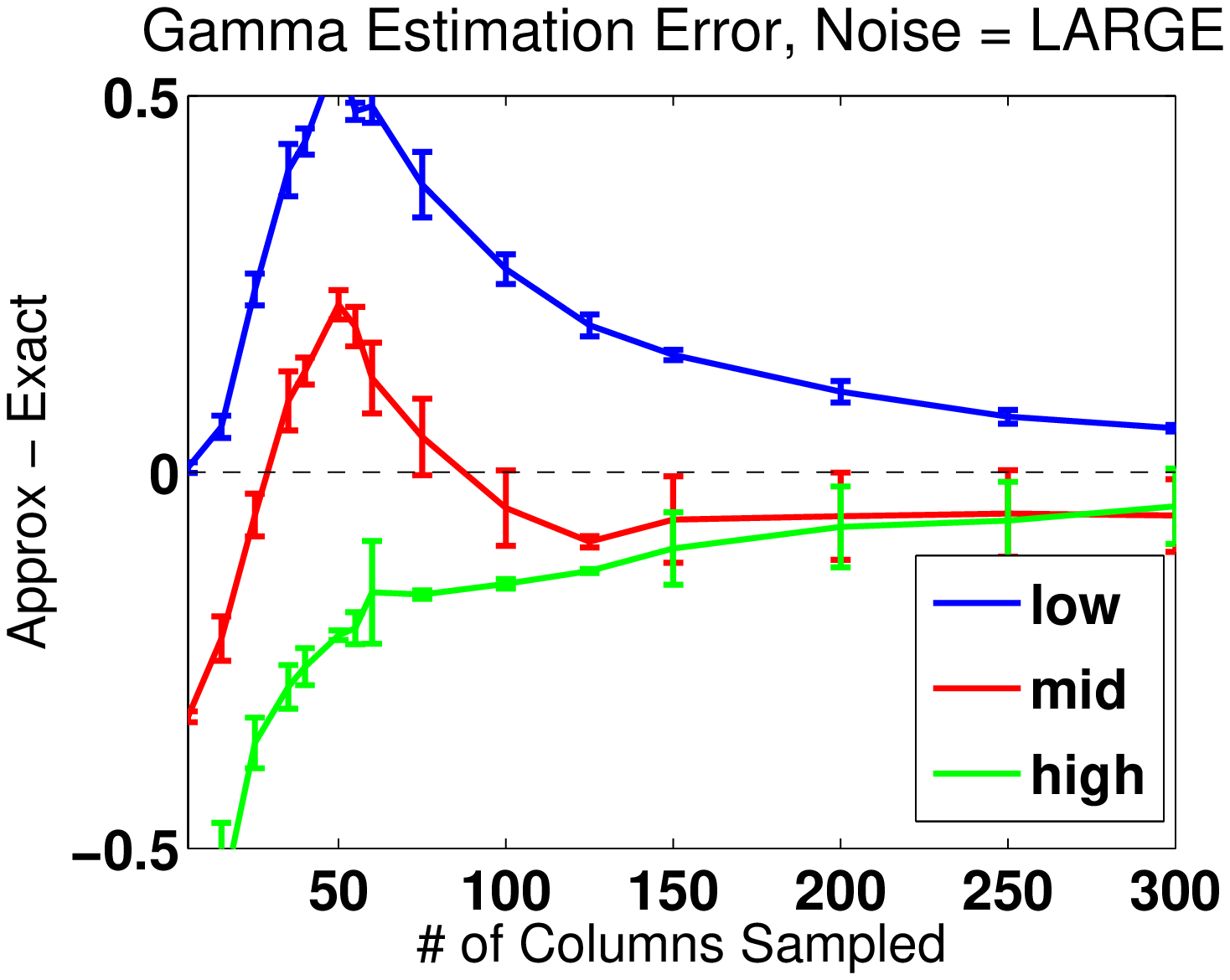} \\
(e) & (f) \\
\end{tabular}
\end{center}
\caption{Experiments with synthetic matrices. (a) True coherence associated
with `low', `mid' and `high' coherences. (b-d) Exact low-rank experiments
measuring difference between the exact coherence and the estimate by
\textsc{Estimate-Coherence}. (e-f) Experiments with low-rank matrices in the
presence of noise, comparing exact and estimated coherence with two different
levels of noise.}
\label{fig:synth_low_rank}
\end{figure}

\subsection{Experiments with synthetic data}
We first generated low-rank synthetic matrices with varying coherence and
singular value spectra, with $n=m=1000$, and $r=50$.  To control the low-rank
structure of the matrix, we generated datasets with exponentially decaying
eigenvalues with differing decay rates, i.e., for $i \in \{1, \ldots, r\}$ we
defined the $i$th singular value as $\sigma_i = \exp(-i\eta)$, where $\eta$
controls the rate of decay and $\eta_{slow} = .01, \, \eta_{medium} = .1, \,
\eta_{fast} = .5$. To control coherence, we independently generated left and
right singular vectors with varying coherences by manually defining one
singular vector and then using QR to generate $r-1$ additional orthogonal
vectors.  We associated this coherence-inducing singular vector with the $r/2$
largest singular value.  We defined our `low' coherence model by forcing the
coherence-inducing singular vector to have minimal coherence, i.e., setting
each component equal to $1 / \sqrt{n}$.  Using this as a baseline, we used $3$
and $8$ times this baseline to generate 'mid' and 'high' coherences (see Figure
\ref{fig:synth_low_rank}(a)).  We then used \textsc{Estimate-Coherence} with
varying numbers of sampled columns to estimate matrix coherence.  Results
reported in Figure \ref{fig:synth_low_rank}(b-d) are means and standard
deviations of $10$ trials for each value of $l$.  Although the coherence
estimate converges faster for the low coherence matrices, the results show that
even in the high coherence matrices, \textsc{Estimate-Coherence} recovers the
true coherence after sampling only $r$ columns.  Further, we note that the
singular value spectrum influences the quality of the estimate.  This
observation is due to the fact that the faster the singular values decay, the
greater the impact of the $r/2$ largest singular value, which is associated
with the coherence-inducing singular vector, and hence the more likely it will
be captured by sampled columns. 

Next, we examined the scenario of low-rank matrices with noise, working with
the `MEDIUM' decaying matrices used in the low-rank experiments.  To create a
noisy matrix from each original low-rank matrix, we first used the QR algorithm
to find a full orthogonal basis containing the $r$ left singular vectors of the
original matrix, and used it as our new left singular vectors (we repeated
this procedure to obtain right singular vectors).  We then defined each of the
remaining $n-r$ singular values of our noisy matrix to equal some fraction of
the $r$th singular value of the original matrix ($0.1$ for `SMALL' noise and
$0.9$ `LARGE' noise). The performance of \textsc{Estimate-Coherence} on these
noisy matrices is presented in Figure \ref{fig:synth_low_rank}(e-f), where
results  are means and standard deviations of $10$ trials for each value of
$l$.  The presence of noise clearly has a negative affect on performance, yet
the estimates are quite accurate for $l=2r$ in the `LOW' noise scenario, and
even for the high coherence matrices with `LARGE' noise, the estimate is fairly
accurate when $l \ge 4r$.

\subsection{Experiments with real data}
We next performed experiments using the datasets listed in Table
\ref{table:coherence_datasets}. We used a variety of kernel functions to
generate SPSD kernel matrices from these datasets, with the resulting kernel
matrices being quite varied in coherence (see Figure \ref{fig:real_data}(a)).
We then used \textsc{Estimate-Coherence} with $r$ set to equal the number of
singular values needed to capture $99\%$ of the spectral energy of each kernel
matrix. Figure \ref{fig:real_data}(b) shows the estimation error over $10$
trials. Although the coherence is well estimated across datasets when $l \ge
100$, the estimates for the two high coherence datasets (\emph{nips} and
\emph{dext}) converge most slowly and exhibit the most variance across trials.
Next, we performed spectral reconstruction using the \nys\ method and matrix
projection reconstruction using the Column-sampling method, and report results
over $10$ trials in Figure \ref{fig:real_data}(c-d).  The results clearly
illustrate the connection between matrix coherence and the quality of these
low-rank approximation techniques, as the two high coherence datasets exhibit
significantly worse performance than the remaining datasets.

\begin{table*}
\small
\centering
\begin{tabular}{|\colspace l \colspace ||\colspace c \colspace |\colspace c
\colspace |\colspace c \colspace |\colspace c \colspace|}
\hline
Dataset & Type of data & \# Points ($n$) & \# Features ($d$) &  Kernel  \\
\hline
NIPS & bag of words & $1500$ &  $12419$ & linear \\ 
PIE & face images & $2731$ &  $2304$ & linear \\ 
MNIS & digit images & $4000$ &  $784$ & linear \\
Essential & proteins & $4728$ & $16$ & RBF \\
Abalone & abalones & $4177$ & $8$ & RBF \\
Dexter & bag of words & $2000$ & $20000$ &  linear \\
KIN-8nm & kinematics of robot arm & $2000$ & $8$ &  polynomial \\
\hline
\end{tabular}
\caption{Description of real datasets used in our coherence experiments,
including the type of data, the number of points ($n$), the number of features
($d$) and the choice of kernel \citep{Sim02,mnist,gustafson05,UCIdatasets}.}
\label{table:coherence_datasets}
\end{table*}

\begin{figure}[ht!]
\begin{center}
\begin{tabular} {@{}c@{}c@{}}
\ipsfig{.35}{figure=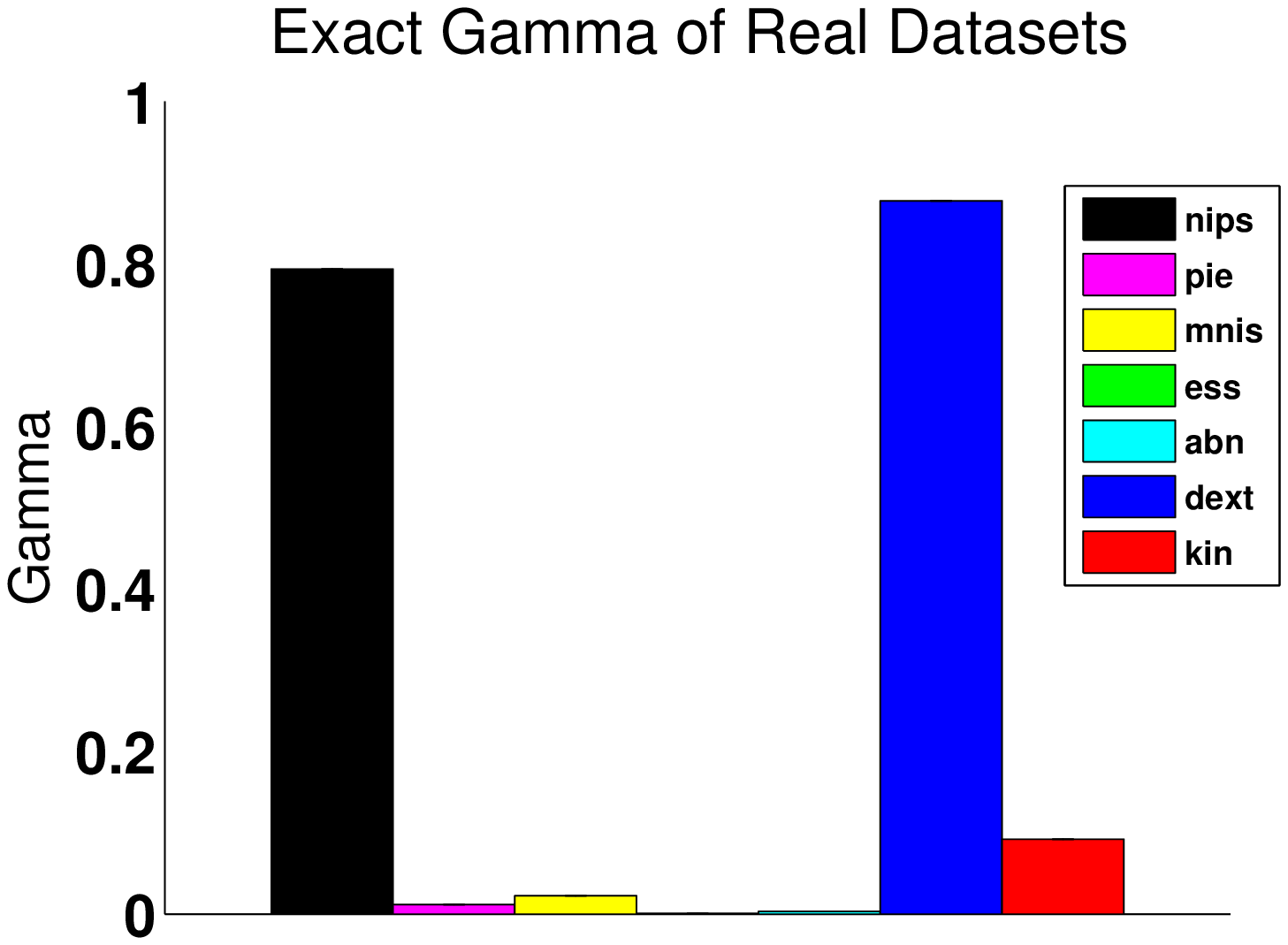} & 
\ipsfig{.35}{figure=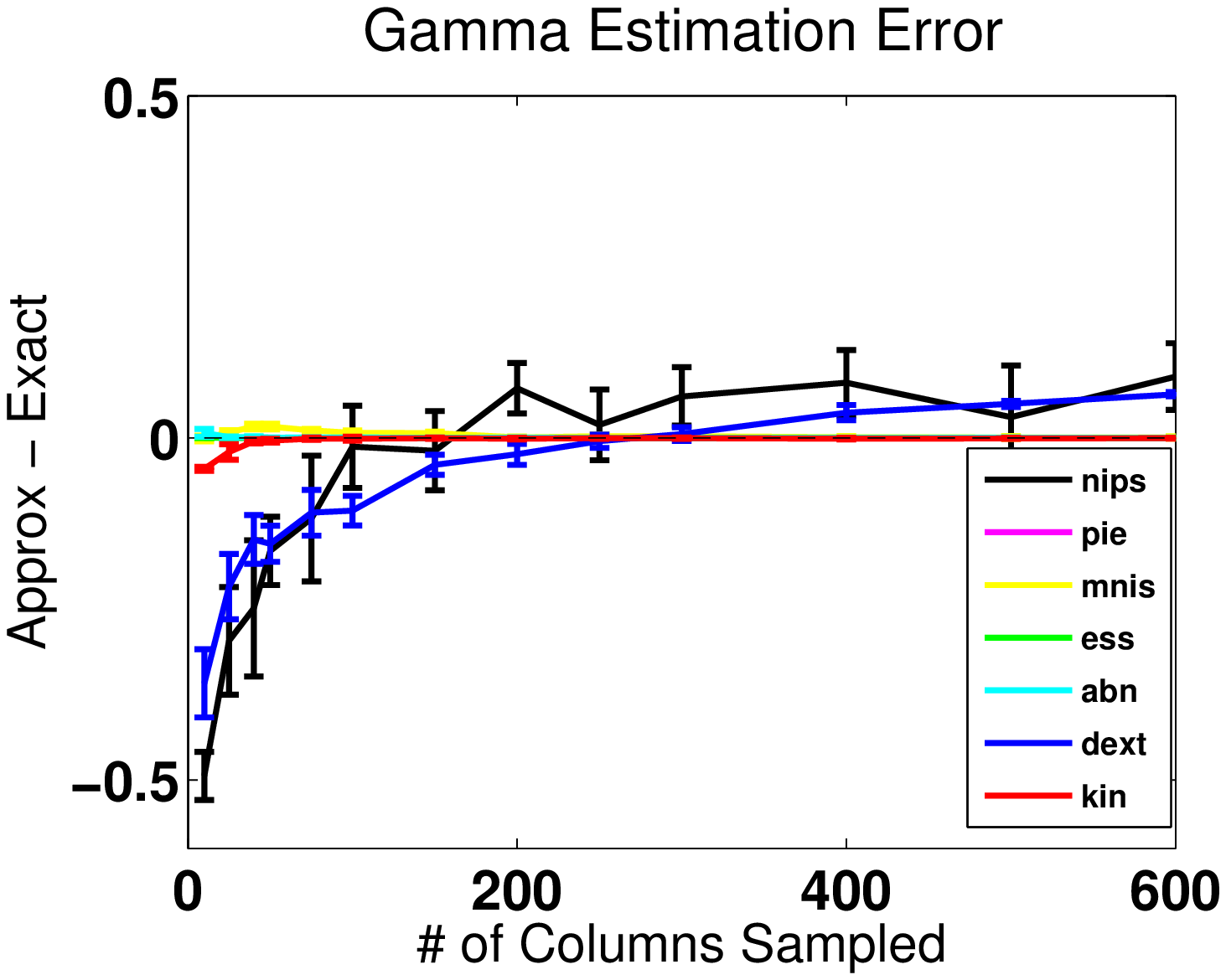} \\
(a) & (b) \\
\ipsfig{.35}{figure=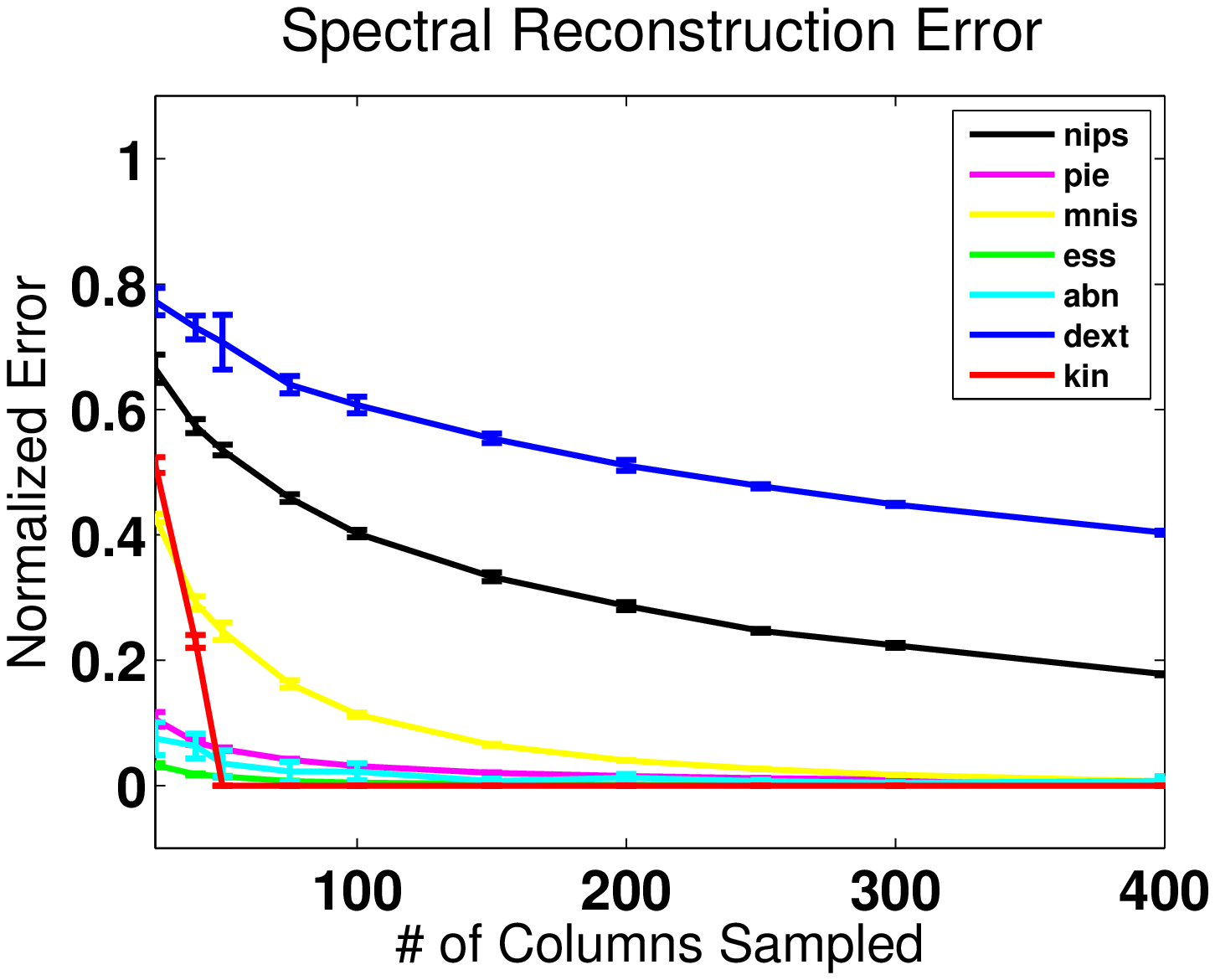} &
\ipsfig{.35}{figure=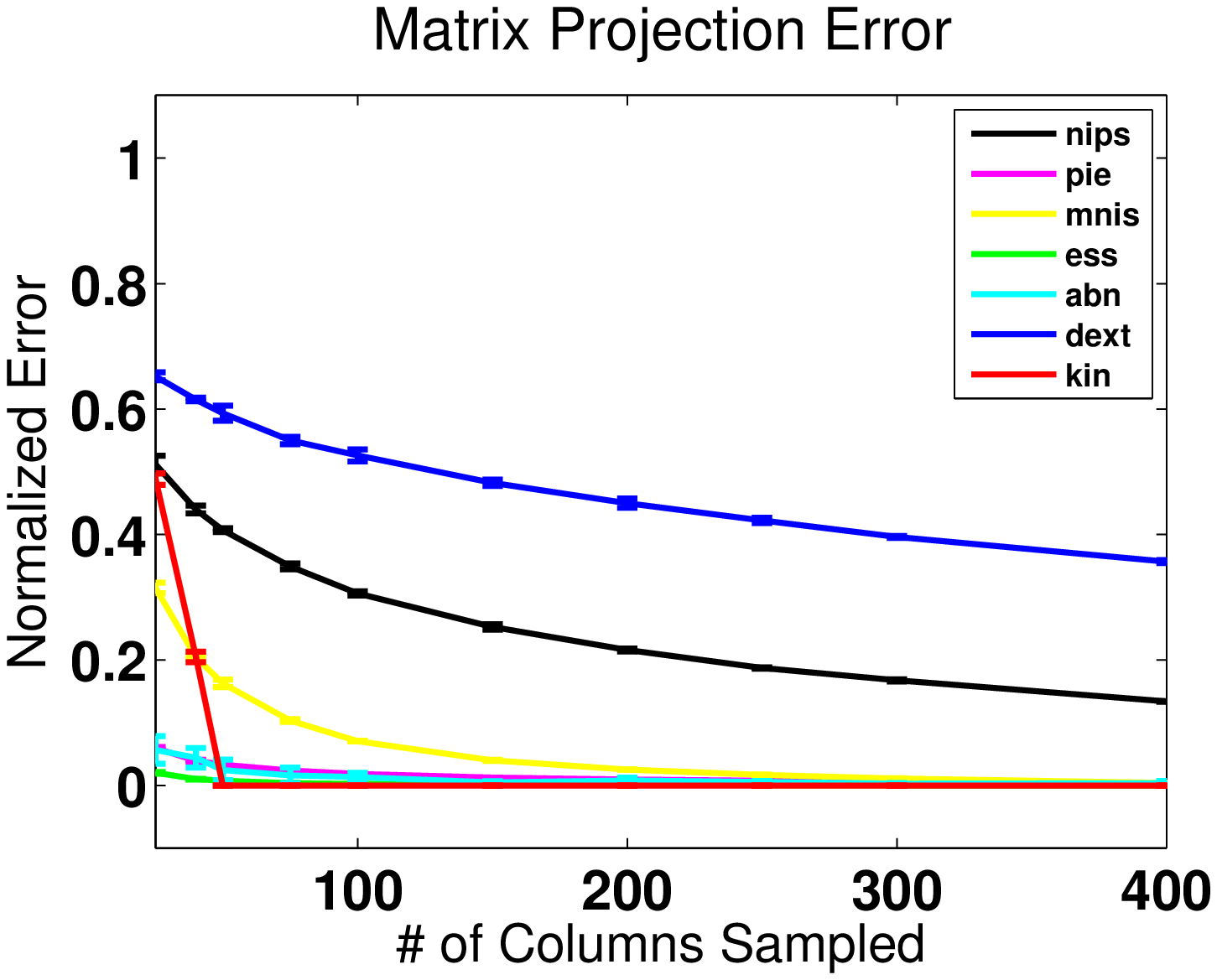} \\
(c) & (d) \\
\end{tabular}
\end{center}
\caption{Experiments with real data. (a) True coherence of each kernel
matrix $\K$. (b) Difference between the true coherence and the estimated
coherence. (c-d) Quality of two types of low-rank matrix approximations ($\tl
\K$), where `Normalized Error' equals $\norm{\K - \tl \K}_F / \norm{\K}_F$.}
\label{fig:real_data}
\end{figure}

\section{Conclusion}
We proposed a novel algorithm to estimate matrix coherence. Our theoretical
analysis shows that \textsc{Estimate-Coherence} provides good estimates for
relatively low-coherence matrices, and more generally, its effectiveness is
tied to coherence itself.  We corroborate this finding for high-coherence
matrices with an adversarially chosen dataset and sampling scheme.
Empirically, however, our algorithm efficiently and accurately estimates
coherence across a wide range of datasets, and these estimates are excellent
predictors of the effectiveness of sampling-based matrix approximation.  We
believe that our algorithm should be used whenever low-rank matrix
approximation is being considered to determine its applicability on a
case-by-case basis.  Moreover, the variance of coherence estimates across
multiple samples may provide further information, and the use of multiple
samples fits nicely in the framework of ensemble methods for low-rank
approximation, e.g., \citet{Kumar09c}. 
 
\small
\bibliographystyle{aaai-named} 
\bibliography{coh} 
\end{document}